\documentclass{article}

\usepackage[utf8]{inputenc}
\usepackage[T1]{fontenc}
\usepackage{hyperref}
\usepackage{url}
\usepackage{booktabs}
\usepackage{amsfonts}
\usepackage{nicefrac}
\usepackage{microtype}
\usepackage{bm}
\usepackage{amsmath}
\usepackage{dcolumn}
\usepackage{longtable}
\usepackage{amsthm}

\usepackage{subfigure}
\usepackage{natbib}
\usepackage{algorithmic}
\usepackage{algorithm}
\usepackage{indentfirst}
\usepackage{graphicx}
\usepackage{color}
\usepackage{comment}
\usepackage{authblk}
\usepackage[papersize={8.5in,11in}]{geometry}
\renewcommand{\algorithmicrequire}{\textbf{Input:}}
\renewcommand{\algorithmicrequire}{\textbf{Output:}}

\newcommand{\argmin}{\mathop{\rm arg~min}\limits}
\setcitestyle{numbers,square,citesep={,},aysep={,},yysep={,}}
\newtheorem{thm}{Theorem}
\newtheorem{lemma}{Lemma}
\newtheorem{definition}{Definition}
\def\atag{\refstepcounter{equation}\tag{\number\c@equation}}
\numberwithin{equation}{section}

\title{LFD-ProtoNet: Prototypical Network Based on Local Fisher Discriminant Analysis for Few-shot Learning}

\author[1]{Kei Mukaiyama}
\author[1,2]{Issei Sato}
\author[1,2]{Masashi Sugiyama}
\affil[1]{The University of Tokyo}
\affil[2]{Riken AIP}
\affil[ ]{{\{mukaiyama@g.ecc.,issei.sato@is.,sugi@k.\}u-tokyo.ac.jp}}

\begin{document}

\maketitle

\begin{abstract}
The prototypical network (ProtoNet) is a few-shot learning framework
that performs metric learning and classification using the distance to
prototype representations of each class.
It has attracted a great deal of attention recently since it is simple
to implement, highly extensible, and performs well in experiments.
However, it only takes into account the mean of the support vectors as
prototypes
and thus it performs poorly when the support set has high variance.
In this paper, we propose to combine ProtoNet with local Fisher
discriminant analysis
to reduce the local within-class covariance and
increase the local between-class covariance of the support set.
We show the usefulness of the proposed method
by theoretically providing an expected risk bound
and empirically demonstrating its superior classification accuracy on
miniImageNet and tieredImageNet.
\end{abstract}

\section{Introduction}
Few-shot learning \cite{Erik,Brenden} is a classification framework from a very small amount of training data.
This framework is used in situations where there is a need to reduce the cost of adding annotations to a large amount of data or there is few data we can use. 
One promising direction to few-shot learning is based on meta learning \cite{Finn}, in which
the training data is separated into a support set for learning representations and a query set for prediction and computing the loss.
This separation unifies the process of learning from the support set and predicting the labels of the query set into a single task.
That is, the problem of few-shot learning is formulated as learning a representation of the support and query sets, called support and query vectors.
The model-agnostic meta learning (MAML) \cite{J.Yoon} learns how to learn by optimizing initial parameters. The matching network (MatchNet) \cite{Vinyals} learns how to add attention or weight from the support set and predicts query labels following the attention mechanism. 
The prototypical network (ProtoNet) \cite{Snell} consists of meta learning and metric learning.
It is simple to implement, highly extensible, and performs as well as complex models in few-shot learning.
The mean vectors of the support vectors are treated as representations for each class, and labels of query vectors are predicted by the distance to the class representations. 
The task adaptive projection network (TapNet) \cite{Sung} is based on ProtoNet and learns class-reference vectors  representing each class.
It also uses singular value decomposition (SVD) to find a subspace onto which the mean vectors and class-reference vectors are projected nearby. 

However, these existing methods still have several drawbacks, leading to undesired classification performance.
ProtoNet only takes into account the mean vectors; thus, it causes misclassification when the variance in the support set is relatively large.
TapNet reduces misalignment of support vectors in the algorithm.
In few-shot learning, however, since the amount of data we can use is small, searching for the best features is difficult.
Thus, TapNet can not find better directions but only removes worse direction, which may result in weak feature extraction. 

In this paper, we propose the use of local Fisher discriminant analysis (LFDA) \cite{Sugiyama} to obtain a feature projection matrix in the feature extraction step in ProtoNet(Fig.\,\ref{figure: net}).
In LFDA, for samples in each class, first the local within-class covariance matrix and the local between-class covariance matrix are computed. Then, it finds \emph{any} number of directions or features that minimize the local within-class covariance and maximize the local between-class covariance. Through this LFDA feature extraction step, we can choose a better subspace from the support set and compute the mean vector for each class after projecting 
them by using the subspace. In the prediction of the query set, we also project query vectors embedded by the network using the subspace obtained in LFDA. By using the mean vectors of the support set and query vectors, we predict the query labels. 
Compared to ProtoNet and TapNet, the remarkable difference is explicitly searching for a better supspace in the algorithm, which leads to significantly better performance in classification.

\textbf{Contributions:}
We make three contributions in this work. 
\begin{enumerate}
\setlength{\parskip}{-0.1cm} 
\item We propose a novel few-shot learning algorithm based on ProtoNet and LFDA, which we refer to as LFD-ProtoNet.
\item We provide an upper bound of the classification risk for LFDA-ProtoNets, which theoretically guarantees the performance of the proposed method. We analyze the effect of the shot number and feature projection matrix.
\item We experimentally show that loss decreases much faster and the accuracy is better than that of TapNet for small iteration complexity.
\end{enumerate}

The code is available online\footnote{Code for LFD-ProtoNet is at \url{https://github.com/m8k8/LFD_ProtoNet.git}. Note: there are some mistakes in our code and the results arer wrong and we'll do experiments again and show correct results.}.

\section{Problem formulation and notations}
\begin{figure}
    \centering
    \includegraphics[width=15cm]{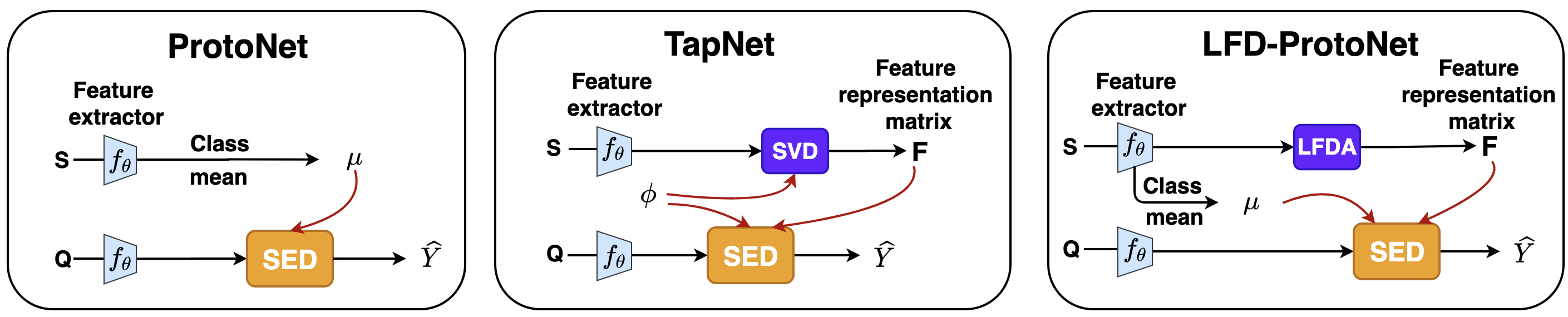}
    \caption{\textbf{Prototypical few-shot learning algorithms.} The representation vectors and feature projection matrix are obtained from the support set $\textbf{\textrm{S}}$, and the query labels of the query set $\textbf{\textrm{Q}}$ are predicted by the squared euclidean distance (SED). The difference of networks is how to obtain the representation vectors and the feature projection matrix.}
    \label{figure: net}
\end{figure}
In few-shot learning, we need to prepare training data with a different configuration from ordinary machine learning \cite{Erik,Brenden}.
A task set is a training data consisting of meta-training data and meta-test data, referred to as a support set and a query set, respectively.

Let $\mathcal{X}$ be the input space and $\mathcal{Y}$ be the output space.
Denote by $x_{c,i}$ the $i$-th input of class $c$.
These samples are drawn independently from an unknown distribution $\mathcal{D}$, i.e., $(x_{c,i},c)\sim \mathcal{D}$.
We define a task as a pair of the support set and query set, and we denote by $\mathcal{T}$ the task distribution.
The task $T=\{D_\mathrm{s},D_\mathrm{q}\}\sim \mathcal{T}$ is composed of support set $D_{s}=\{(x_{c,i}\in \mathcal{X},c\in \mathcal{Y})_{c,i}\}$ and query set $D_{q}=\{(x'_{c,i}\in \mathcal{X},c\in \mathcal{Y})_{c,i}\}$. 
The number of support sets is limited such that when we classify $C$ classes in $k$-shot learning, the number of the samples per class is $k$. 

Denote by $m$ the dimension of the embedded latent feature space and by $n$ that of the space projected by feature projection matrix.
We define the network $f_{\theta}:\mathcal{X}\rightarrow \mathbb{R}^m$ that embeds an input sample $x_{c,i}\in \mathcal{X}$ to a latent feature space space $\mathbb{R}^m$ with the parameter $\theta$.
For learning representations, we define the feature extractor $\mathrm{FE}$ that is given some information such as embedded support vectors or support labels and returns a feature projection matrix $F\in \mathbb{R}^{m\times n}$.
It is expected that $F$ extracts better statistical information from $m$-dimensional latent feature space to $n$-dimensional latent feature space.
Let $\overline{f_{\theta}(x_{c,\cdot})}\in \mathbb{R}^n$ be a representation vector for each class $c$ to predict the label of the query vector with the Euclidean distance.

Finally, we define the loss function $g:\mathbb{R}^n\times \mathbb{R}^n\times \mathbb{R}^{(k-1)\times n}\rightarrow \mathbb{R}$ and a generalization loss of the network $f_{\theta}$ as follows. 
Let $v_c$ be  the query vector that is embedded by $f_{\theta}$ and projected by $F$.
\begin{align}
g(v_c,\overline{f_{\theta}(x_{c,\cdot})},\{\overline{f_{\theta}(x_{s,\cdot})}\}_{s\neq c}):=d(v_c,\overline{f_{\theta}(x_{c,\cdot})})+\log \sum_{s\neq c} \exp (-d(v_c,\overline{f_{\theta}(x_{s,\cdot})})),
\end{align}
where  the distance function $d(\cdot)$ is given by the L2 norm, i.e, $d(v_c,\overline{f_{\theta}(x_{s,\cdot})})=\|v_c-\overline{f_{\theta}(x_{s,\cdot})}\|_2^2$ and $\{\overline{f_{\theta}(x_{s,\cdot})}\}_{s\neq c}$ indicates all class representation vectors excluding class $c$. 

The generalization loss $L_{\mathcal{T}}(f)$ is defined by
\begin{align}
L_{\mathcal{T}}(f):=\mathbb{E}_{T\sim \mathcal{T}}[g(Ff(x'_c),\overline{f_{\theta}(x_{c,\cdot}))},\{\overline{f_{\theta}(x_{s,\cdot})}\}_{s\neq c}].
\end{align}
The empirical loss $\widehat{L}_{\mathcal{T}}(f)$ is also defined by
\begin{align}
\widehat{L}_{\mathcal{T}}(f):=\frac{1}{N}{\displaystyle \sum_{T_u\sim \mathcal{T},1\leq u \leq N}}\frac{1}{CM}\sum_{c=1}^C\sum_{i=1}^Mg(Ff(x'_{c,i}),v_b,\{\overline{f_{\theta}(x_{s,\cdot})}\}_{s\neq b}),
\end{align}
where $N$ is the number of tasks, $M$ is the number of query data per class, and $C$ is the number of classes.

We construct a hypothesis set $\mathcal{H}=\{f_{\theta} | \theta \in \Omega\}$ and aim to solve the following optimization problem:
\begin{align}
f^*=\argmin_{f\in \mathcal{H}} L_{\mathcal{T}}(f).
\end{align}
However, since $L_{\mathcal{T}}$ is not accessible, we minimize the empirical risk in practice:
\begin{align}
\widehat{f^*}=\argmin_{f\in \mathcal{H}} \widehat{L}_{\mathcal{T}}(f).
\end{align}
We explain existing work on the basis of this problem formulation as follows.

\textbf{The prototypical network (ProtoNet)} \cite{Snell} is a pioneering algorithm of combining meta learning and few-shot learning.
A number of variants based on this algorithm have been proposed because of its simplicity and better performance \cite{Sung,Oreshkin}.
The simplicity of ProtoNet lies in making the feature projection matrix $F_{\textrm{Proto}}$ an identity matrix, i.e., $F_{\textrm{Proto}}=I_m(m=n)$, in the feature extraction step. 
Moreover, what ProtoNet requires for obtaining the $c$-th class representation $\overline{f_{\theta}(x_{c,\cdot})}$ is only to average support vectors, which are embedded by $f_{\theta}$, belonging to class $c$. 
That is, $\overline{f_{\theta}(x_{c,\cdot})}$ is simply given by
\begin{align}
\overline{f_{\theta}(x_{c,\cdot})}=\frac{1}{k} \sum_{i=1}^k f_{\theta}(x_{c,i}).
\end{align}

\textbf{The task adaptive projection network (TapNet)} \cite{Sung} adds a feature projection function to ProtoNet. 
As with ProtoNet, the training data is split into the support set and the query set.
Then, a class reference vector $\phi_c$ is introduced to represent each class $c$.
In the feature extraction step of TapNet, a matrix projecting the class reference vectors $\phi_c$ and the mean vectors per class $\frac{1}{k}\sum_{i=1}^{k} f(x_{c,i})$ to the same location for each class is found where $k$ is the number of the data of class $c$ and $\overline{f_{\theta}(x_{c,\cdot})}$ indicates that $x_c$ belongs to class $c$. 
When the norm of all vectors is ignored and only their directions are considered, the matrix $F_{\textrm{SVD}}$ is obtained by the following equation with singular value decomposition:
\begin{align}
\forall c \in \{ 1,2,3\ldots,C \}.~F_{\textrm{SVD}}\frac{\overline{f_{\theta}(x_{c,\cdot})}}{\| \overline{f_{\theta}(x_{c,\cdot})} \|_2} =F_{\textrm{SVD}}\frac{\phi_c}{\| \phi_c \|_2}.
\end{align}
The query data is similarly embedded by $f_{\theta}$ and projected by $F_{\textrm{SVD}}$.
Finally, their labels are estimated by the distances between $F_{\textrm{SVD}} f(x_{c,i})$ and $F_{\textrm{SVD}}\phi_k$.

Following Chen et al.\cite{Chen}, the schematics of ProtoNet and TapNet are illustrated in Fig.\,\ref{figure: net}.

\section{Preliminaries}
\label{prelim}
\textbf{Fisher discriminant analysis (FDA)} \cite{Fisher} is used for finding the subspace to make the within-class covariance $S_{\mathrm{wit}}$ small and between-class covariance $S_{\mathrm{bet}}$ large:
\begin{align}
S_{\mathrm{wit}} = \frac{1}{kC} \sum_{c=1}^C\sum_{i=1}^k(x_{c,i} - \mu_c)(x_{c,i}-\mu_c)^{\top},
\quad
S_{\mathrm{bet}} = \frac{1}{C} \sum_{c=1}^C(\mu_c - \mu)(\mu_c -\mu)^{\top},
\end{align}
where $k$ is the number of samples in each class, $\mu_c$ is the mean vector of the class $c$, and $\mu$ is the mean vector of all samples in $\mathbb{R}^n$. We take $\bm{w}$ as the directions to project samples and minimize the ratio of the within-class covariance of vectors projected by $\bm{w}$ and their between-class covariance:
\begin{align}
\min_{\bm{w}} \frac{\bm{w}^{\top}S_{\mathrm{wit}}\bm{w}}{\bm{w}^{\top}S_{\mathrm{bet}}\bm{w}}.
\end{align}
The $\bm{w}$ is the best direction that minimizes the ratio.
We can also find some directions that also make the ratio small if the problem is a multi-class case. To obtain these directions, we solve a modified problem
\begin{align}
W_{\mathrm{FDA}}=\argmin_{W} \mathrm{Tr}((W^{\top}S_{\mathrm{bet}}W)^{-1}W^{\top}S_{\mathrm{wit}}W),
\end{align}
where $W$ is the $n\times (C-1)$ matrix and $\mathrm{Tr}(\cdot)$ is the trace of the matrix. 
Since the solution is the eigenvectors of $S_{\mathrm{bet}}^{-1}S_{\mathrm{wit}}$, the rank of this matrix is just $C$, which is the number of the class.
That is, the total number of eigenvectors is also just $C$, which is a problem with few-shot learning described in Sec.\ref{sec:prop}.

\textbf{Local Fisher discriminant analysis (LFDA)} \cite{Sugiyama} is an extension of FDA.
When samples in a class are multimodal, keeping local within-class scatter can be hard in FDA because multimodal samples should be merged into a single cluster.
This constraint results in less separate embedding due to less degree of freedom. 
To solve this problem, LFDA combined FDA and \textit{locality-preserving projection} \cite{Niyogi} and construct the within-class covariance $S_{\mathrm{wit}}^{A}$ and between-clss covariance $S_{\mathrm{bet}}^{A}$ by using the affinity matrix $A$ whose elements are the similarities of samples in the same class, e.g.,
the squared exponential kernel $\exp(-(||x_i-x_j||)^2)$ is used for $A_{i,j}$.
The details of $S_{\mathrm{wit}}^{A}$ and $S_{\mathrm{bet}}^{A}$ are in Appendix \ref{Sec:LFDA} in the supplementary material.
The objective function to minimize is the same, i.e.,
\begin{align}
W_{\mathrm{LFDA}}=\argmin_{W}\mathrm{Tr}((W^{\top}S_{\mathrm{bet}}^{A}W)^{-1}W^{\top}S_{\mathrm{wit}}^{A}W).
\end{align}
In FDA, the rank of $(S_{\mathrm{bet}})^{-1}S_{\mathrm{wit}}$ is just the number of classes $C$; however, in LFDA the rank of $(S_{\mathrm{bet}}^{A})^{-1}S_{\mathrm{wit}}^{A}$ is the number of samples $kC$ because by adding the similarity terms, vectors with a linear dependency in FDA has a linear dependency.

\section{Proposed method}
\label{sec:prop}
Since the feature projection matrix in ProtoNet is just an identity matrix $I_n$, ProtoNet uses no information about the support set.
That is, incorporating the support set can improve the classification performance of few-shot learning based on ProtoNet. 
The feature projection matrix $F_{\textrm{SVD}}$ of TapNet aims to reduce the misalignment of support vectors by eliminating worse directions up to $C$. 
This means that if we have much more support data, we can reduce misalignment more; however, the more support data we obtain, the less features we can use. 
This is counterintuitive because the ideal situation is that if we obtain more data for the support set, then we can obtain more features for each class. 
In this section, we propose a novel feature projection matrix in accordance with the intuition that more data lead to more useful features.

\subsection{Algorithm}
\label{sec:alg}
We suppose that if we can make the local within-class covariance of the support set smaller and at the same time can make its local between-class covariance larger, then the classification performance is expected to be improved. 
This concept was originally introduced in FDA \cite{Fisher}; hence, using FDA for feature extraction is one option. 
The dimension of the features extracted by FDA, however, is limited to the rank of the covariance matrix, i.e., $C-1$ as described in Sec.\,\ref{prelim}.
That is, the expression power of the FDA features is typically insufficient.
 To solve this problem, we propose to use LFDA, in which we can increase the dimension of the extracted features to $kC-1$  as described in Sec.\,\ref{prelim}.
 LFDA can usually extract any number of the features up to the number of the support set so we can use more features from the support set.
 
By using the feature projection matrix $ F_{\textrm{LFDA}}^{c}$ in LFDA (see Eq.\,\eqref{eq:lfd} below), we formulate the representation vector $\overline{f_{\theta}(x_{c,\cdot})}$ as the mean vectors of $F_{\textrm{LFDA}}^{c}f_{\theta}(x_{c,i})$ in terms of $i$, i.e.,
\begin{align}
\overline{f_{\theta}(x_{c,\cdot})}=\frac{1}{k}\sum_{i=1}^{k}F_{\textrm{LFDA}}^{c}f_{\theta}(x_{c,i}).
\end{align}
The query vectors are embedded by $f_{\theta}$, projected by $F_{\textrm{LFDA}}^{c}$, and predicted as the class $c$ that is the class of the nearest representation vector to the query vector.
We summarize the proposed algorithm in Algorithm\,\ref{fig: bo}.
\begin{algorithm}[t!]
    \renewcommand{\algorithmicrequire}{\textbf{Input:}}
    \renewcommand{\algorithmicensure}{\textbf{Output:}}
    \caption{Few-shot learning (k-shot) algorithm framework based on ProtoNet}
    \label{fig: bo}
    \textbf{Notations:} Denote by $T$ a task, drawn from $\mathcal{T}$, compased of support set $D_s$ and query set $D_q$. Denote by $L_{tr}$ a training loss, by $F_{\mathrm{LFDA}}^{c}$ a feature projection matrix of local Fisher discriminant analysis, and by $f_{\theta}$ an embedding function with parameters $\theta$. In this algorithm $C$ means the number of classes in the task and $M$ means the number of query samples per class.
    \begin{algorithmic}[1]
    \REQUIRE training task $\{T_u\}_{u=1}^N \sim \mathcal{T}$ where $T_u=\{D_s,D_q\}$, $D_s=\{(x_{c,i},c)\}_{1\leq c \leq C,1\leq i \leq k}$ and $D_q=\{(x'_{c,i},c)\}_{1\leq c \leq C,1\leq i \leq M}$.
    \STATE $L_{tr}\leftarrow 0$
    \FOR{$u$ in $u=0,1,\ldots,N$}
    \STATE $(D_s,D_q)\leftarrow T_u$
    \STATE $F_{\mathrm{LFDA}}^{c}=\argmin_{W} \mathrm{Tr}((W^{\top}\Sigma_{F,c}W)^{-1}W^{\top}\Sigma_FW)$, $c=1,\ldots,C$ using Eqs.\,\eqref{eq:sigmaf} and \eqref{eq:sigmafc}.
    \STATE $\{\overline{f_{\theta}(x_{c,\cdot})}\}_{c=1}^C=\{\frac{1}{k}\sum_{i=1}^k F_{\textrm{LFDA}}^{c}f_{\theta}(x_{c,i})\}_{c=1}^C$
    \STATE $L_{T_u}\leftarrow 0$
    \FOR{$c$ in $c=0,1,\ldots,C$}
    \FOR{$i$ in $i=0,1,\ldots,M$}
    \STATE $L_{T_u} \leftarrow L_{T_u} + g(F_{\textrm{LFDA}}^{c}f_{\theta}(x'_{c,i}),\overline{f_{\theta}(x_{c,\cdot})},\{\overline{f_{\theta}(x_{s,\cdot})}\}_{s\neq c})$
    \ENDFOR
    \ENDFOR
    \STATE $L_{tr}\leftarrow L_{tr} + \frac{1}{CM}L_{T_u}$
    \ENDFOR
    \STATE $L_{tr}\leftarrow \frac{1}{N}L_{tr}$
    \STATE update $\theta$ with $L_{tr}$
    \end{algorithmic}
\end{algorithm}

\subsection{Theoretical analysis}
\label{sec:theory}
The effect of $k$ in ProtoNet was analyzed by Cao et al. \citep{Tianshi}. 
We analyze our algorithm in line with their work and  show that how our algorithm is theoretically better than ProtoNet.

As in Cao et al. \citep{Tianshi}, we first consider, for simplicity, the case where the query is the binary classification of class $\mathrm{a}\in \mathcal{Y}$ or $\mathrm{b}\in \mathcal{Y}$. The result can be easily generalized to the multi-class classification (see Appendix \ref{ap:multi}).
These $\mathrm{a}$ and $\mathrm{b}$ are random variables from all class sets $\mathcal{Y}$. The support set of $\mathrm{a}$ is defined as $\mathcal{S}_{\mathrm{a}}=\{x_{\mathrm{a},i}\}_{i=1}^k$, and that of $b$ is defined as $\mathcal{S}_{\mathrm{b}}=\{x_{\mathrm{b},i}\}_{i=1}^k$. 
The whole support set is $S=\{\mathcal{S}_{\mathrm{a}},\mathcal{S}_{\mathrm{b}}\}$. 
The support sets $\mathcal{S}_{\mathrm{a}}$ and $\mathcal{S}_{\mathrm{b}}$ are embedded by the network $\phi$ and a feature projection matrix $F$ is obtained by LFDA to make the local between-class covariance large and the local within-class covariance small. With LFDA, we get the representation vectors of classes $\mathrm{a}$ and $\mathrm{b}$. When we take $x\in \mathcal{X}$ from the query set, it is also embedded by $\phi$, projected by $F$, and finally the distances to $\mathrm{a}$ and $\mathrm{b}$ are compared. In the analysis below, we assume that the query $x$ belongs to class $\mathrm{a}$. 
\begin{definition}[Representation vector of class]
We define $\overline{\phi(\mathcal{S}_c)}$ as the mean vector of support vectors in class $c\in \{\mathrm{a},\mathrm{b}\}$ so the representation vector of class $c$ is written as $\overline{\phi(\mathcal{S}_c)}$. In the $k$-shot learning $|\mathcal{S}_c|=k$, we can write $\overline{\phi(\mathcal{S}_c)}$ as
\begin{align}
\overline{\phi(\mathcal{S}_c)}=\frac{1}{k}\sum_{i=1}^k\phi(x_{a,i}).
\end{align}
We also define $\overline{F\phi(\mathcal{S}_c)}$ as the mean vector of support vectors projected by feature projection matrix $F$ in class $c$.
The representation vector of class $c$ after the feature extraction step is written as
\begin{align}
\overline{F\phi(\mathcal{S}_c)}=\frac{1}{k}\sum_{i=1}^kF\phi(x_{a,i}).
\end{align}
\end{definition}

\begin{definition}[Between-class covariance and within-class covariance]
We define between-class covariance matrix $\Sigma$ and within-class covariance matrix of class $c$ $\Sigma_c$ as
\begin{align}
\Sigma&=\frac{1}{2}((\overline{\phi(\mathcal{S}_a)}-\overline{\phi(\mathcal{S})})(\overline{\phi(\mathcal{S}_a)}-\overline{\phi(\mathcal{S})})^{\top}+(\overline{\phi(\mathcal{S}_b)}-\overline{\phi(\mathcal{S})})(\overline{\phi(\mathcal{S}_b)}-\overline{\phi(\mathcal{S})})^{\top}),\\
\Sigma_c &=\frac{1}{k}\sum_{i=1}^k(\phi(x_{c,i})-\overline{\phi(\mathcal{S}_c)})(\phi(x_{c,i})-\overline{\phi(\mathcal{S}_c)})^{\top}.
\end{align}
We also define $\Sigma_{F}$ and $\Sigma_{F,c}$ as both the between-class covariance matrix projected by $F$ and within-class covariance matrix projected by $F$ in class $c$ as follows.
\begin{align}
\Sigma_{F}=&\frac{1}{2}\Bigl\{\left(\overline{F\phi(\mathcal{S}_a)}-\overline{F\phi(\mathcal{S})}\right)\left(\overline{F\phi(\mathcal{S}_a)}-\overline{F\phi(\mathcal{S})}\right)^{\top}\notag\\
&+\left(\overline{F\phi(\mathcal{S}_b)}-\overline{F\phi(\mathcal{S})}\right)\left(\overline{F\phi(\mathcal{S}_b)}-\overline{F\phi(\mathcal{S})}\right)^{\top}\Bigr\}\label{eq:sigmaf},\\
\Sigma_{F,c} =&\frac{1}{k}\sum_{i=1}^k\left(F\phi(x_{c,i})-\overline{F\phi(\mathcal{S}_c)}\right)\left(Ff(x_{c,i})-\overline{F\phi(\mathcal{S}_c)}\right)^{\top}\label{eq:sigmafc}.
\end{align}
\end{definition}

\begin{definition}[Task loss]
The task loss $\ell_{\mathrm{task}}(T)$ with $0$-$1$ loss $\ell_{err}$ is defined as
\begin{align}
\ell_{\mathrm{task}}(T) = \frac{1}{MC}\sum_{i=1}^{MC} \ell_{\mathrm{err}}(\widehat{y_i},y_i),
\end{align}
where $T=\{D_s,D_q\}\sim \mathcal{T}$, $D_\mathrm{s}$ and $D_\mathrm{q}$ are the support set and query set, and $\widehat{y_i}$ and $y_i$ are the $i$-th estimated label and the $i$-th true label in the query set $D_q$.
\end{definition}

\begin{definition}[Empirical risk of $\phi$]
We define the empirical risk of $\phi$ using task loss $\ell_{\mathrm{task}}$ where $m$ tasks $T_u$ are drawn independently from the task distribution $\mathcal{T}$, i.e.,
\begin{align}
\widehat{R}_{n,\mathcal{T},c}(\phi) = \frac{1}{N}\sum_{u=1}^N\ell_{\mathrm{task}}(T_u),\quad T_u\sim \mathcal{T}~(u=1,\ldots,N).
\end{align}
\end{definition}

\begin{definition}[Expected risk of $\phi$]
\label{def:ex_risk}
We define the risk of $\phi$ using the expectation in terms of the task distribution $\mathcal{T}$ as
\begin{align}
R_{\mathcal{T},c}(\phi)=\mathbb{E}[\widehat{R}_{n,\mathcal{T},c}(\phi)].
\end{align}
\end{definition}

\begin{thm}[Upper-bound of expected risk with LFDA]\label{maintheorem}
Consider $k$-shot learning. Under the same assumptions as Cao et al.\cite{Tianshi}, in which $\Sigma_a=\Sigma_b$ and $p(\phi(X)|Y(X)=c)$ is the Gaussian distribution with mean $\mu_c$ and variance $\Sigma_c$, i.e., $\phi(X)|Y(X)=c \sim \mathcal{N}(\mu_c,\Sigma_c)$,  the expected risk of $\phi$ with the $0$-$1$ loss is bounded as
\begin{align}
R_{\mathcal{T},c}(\phi) \leq 1 - \frac{4\mathrm{Tr}\left(\Sigma_F\right)^2}{8\left(1+\frac{1}{k}\right)^2\mathrm{Tr}\left(\Sigma_{F,c}^2\right)+16\left(1+\frac{1}{k}\mathrm{Tr}(\Sigma_F \Sigma_{F,c})\right) + \mathbb{E}[\left((\mu_a-\mu_b)^{\top}F^{\top}F(\mu_a-\mu_b)\right)^2]}.\label{eq:main}
\end{align}
\end{thm}

A Proof of Theorem \ref{maintheorem} is given in Appendix \ref{ap:proof}.
The numerator $4\mathrm{Tr}(\Sigma_F)^2$ is $\mathcal{O}(\Sigma_F^2)$, the first term in the denominator $8(1+\frac{1}{k})^2\mathrm{Tr}(\Sigma_{F,c}^2)$ is $\mathcal{O}(\Sigma_{F,c}^2)$, the second term $16(1+\frac{1}{k}\mathrm{Tr}(\Sigma_F \Sigma_{F,c}))$ is $\mathcal{O}(\Sigma_F\Sigma_{F,c})$, and the last term $\mathbb{E}[((\mu_a-\mu_b)^{\top}F^{\top}F(\mu_a-\mu_b))^2]$ is $\mathcal{O}(\Sigma_F^2)$. For the last term, if we assume that $F$ satisfies the conservation of the norm, it is clear that $(\mu_a-\mu_b)^{\top}F^{\top}F(\mu_a-\mu_b)$ becomes large when the between-class covariance is relatively large. Thus we can conclude that if $\mathrm{Tr}(\Sigma_{F}^{-1}\Sigma_{F,c})$ is small, the right-hand side of the inequality in (\ref{eq:main}) becomes small so that the risk will be close to zero. 
Moreover, LFDA tries to find the subspace that makes $\mathrm{Tr}(\Sigma_F^{-1}\Sigma_{F,c})$ minimum, i.e.,
\begin{align}
\label{eq:lfd}
F_{\mathrm{LFDA}}^{c}=\argmin_{W} \mathrm{Tr}((W^{\top}\Sigma_{F}W)^{-1}W^{\top}\Sigma_{F,c}W).
\end{align}
Thus, we can expect that  LFD-ProtoNet performs better than ProtoNet since $F$ is an identity matrix in ProtoNet.

\section{Experiment}
In Section \ref{sec:theory}, we showed that our algorithm improves the upper bound of the risk if $\Sigma_F^{-1}\Sigma_{F,c}$ is smaller than $\Sigma^{-1} \Sigma_c$.
In this experiment, we check how much better the performance of our algorithm compared to other few-shot methods.
We also did experiment for comparing the trace value of $\Sigma_{F}^{-1}\Sigma_{F,c}$ in LFD-ProtoNet and $\Sigma^{-1} \Sigma_c$ in ProtoNet.
\subsection{Dataset}
We used two benchmark datasets well-used in few-shot learning.
\paragraph{miniImageNet\citep{Vinyals}}
This dataset is a subset of the ILSVRC-12 ImageNet datas\citep{Russakovsky} with 100 classes and $600$ images per class. In this setting, the size of  the images is $84$ $\times$ $84$. And the training data contains $64$ classes, the validation data contains $16$ classes, and the test data contains $20$ classes.
\paragraph{tieredImageNet\citep{Ren}}
This dataset is a larger set than miniImageNet with $608$ classes and $779,165$ images. It has $34$ categories, and these categories are split into $20$ training, $6$ validation, and $8$ test categories.

\subsection{Implementation detail}
We performed an experiment on $5$-shot and $1$-shot case with miniImageNet and tieredImageNe. 
The number of the iterations was $40,000$ and we used ResNet-$12$\citep{K.He} as the network where the output dimension was $128$. We generated training data, validation data, and test data with a data generator. As the first step, the training data was split into the support set and query set. Then, we made many tasks that contain support data for feature extraction and query data for computing loss function. In this experiment, we used cross-entropy loss as the loss function. For each task, the loss was calculated, and the parameters of the network or embedding function were updated in the training step. In $5$-shot learning, we obtained $5$ samples for each class. That is, the number of all samples for feature extraction was $24$ in the $5$-class classification problem. As a property of the features, from $n$ samples, we can obtain at most $n-1$ features, which means that we can obtain by LFD at most $24$ features in this setting. Similarly in $1$-shot learning, we can obtain $4$ features from the support set.

We also performed an experiment with LFDA in the $5$-shot case of miniImageNet and compared the performance of the LFDA and FDA cases in respect of the number of features from the support set.

As we showed in Section \ref{sec:theory}, small $\mathrm{Tr}(\Sigma_F^{-1}\Sigma_{F,c})$ is preferable and we performed an experiment comparing the value $\mathrm{Tr}(\Sigma_F^{-1}\Sigma_{F,c})$ with $\mathrm{Tr}(\Sigma^{-1}\Sigma_c)$ of ProtoNet. This result is shown in the supplementary material due to the lack of space.

\subsection{Results}
Table\ref{tab:result:miniImage_tieredImage} shows that LFD-ProtoNet achieved $-\%$ on the miniImageNet($1$-shot), $-\%$ on the miniImageNet($5$-shot), $-\%$ on the tieredImageNet($1$-shot), and $-\%$ on the tieredImageNet($5$-shot). It is clear that the method with LFDA achieves the best performance of all other methods, and this is because searching for the best subspace to project positively is superior to reducing worse directions such as TapNet. In the case of $1$-shot learning, the number of samples for each class is exactly $1$ and this fact means that FDA can only consider the within-class covariance. However, it is sufficient for LFD-ProtoNet to outperform others only with the local within-class covariance and it also shows that searching for the better subspace makes sense. 

We show the results of the case of FDA. In $5$-shot learning, when we use FDA as the feature extractor, the accuracy is only $-\%$, which is the almost same as that of adaResNet. It can be considered that FDA returns features up to only the number of classes; thus, if total classes are $C$, we can extract only $C-1$ features, and it is insufficient for the training.  If we use LFDA, However, we can extract features up to the number of samples.
Therefore, in the $k$-shot case, the number of samples is $kC-1$ in the training step and it is sufficient for the network to learn.

We measured loss decreasing speed of TapNet and LFD-ProtoNet. In TapNet, the training loss decreased slowly up to $30,000$ epochs. This can result in overfitting to training data. However, in LFD-ProtoNet, the training loss quickly decreased for $10,000$ epochs and this fact can be thought as fast adaptation without overfitting.

\begin{table}[!t]
    \caption{The result of the few-shot learning experiment with miniImageNet and tieredImageNet. The N/A indicates ``not available in the original paper''.}\label{tab:result:miniImage_tieredImage}
    \centering
    \renewcommand{\arraystretch}{1.2}
    \begin{tabular}{l||c|c||c|c}
    \toprule
    Method & \multicolumn{2}{|c||}{miniImageNet} & \multicolumn{2}{|c}{tieredImageNet}\\
    \hline
    & $1$-shot & $5$-shot & $1$-shot & $5$-shot \\
    \midrule
    Matching Nets \citep{Vinyals} & $43.56 \pm 0.84\%$ & $55.31 \pm 0.73\%$ & N/A & N/A \\
    MAML \citep{Finn} & $48.70 \pm 1.84\%$ & $63.15\pm 0.91\%$ & $51.67 \pm 1.81\%$ & $70.30 \pm 1.75\%$\\
    ProtoNet \citep{Snell} & $49.42 \pm 0.78\%$ & $68.20 \pm 0.66\%$ & $53.31 \pm 0.89\%$ & $72.69 \pm 0.74\%$\\
    SNAIL \citep{Mishra} & $55.71 \pm 0.99\%$ & $68.88 \pm 0.92\%$ & N/A & N/A\\
    adaResNet \citep{Munkhdalai} & $56.88 \pm 0.62\%$ & $71.94 \pm 0.57\%$ & N/A & N/A\\
    TPN \citep{Liu} & $55.51 \pm 0.86\%$ & $69.86 \pm 0.65\%$ & $59.91 \pm 0.94\%$ & $73.30 \pm 0.75\%$ \\
    TADAM-$\alpha$ \citep{Oreshkin} & $56.8 \pm 0.3\%$ & $75.7 \pm 0.2\%$ & N/A & N/A\\
    TADAM-TC \citep{Oreshkin} & $58.5 \pm 0.3\%$ & $76.7 \pm 0.3\%$ & N/A & N/A\\
    Relation Nets \citep{Yang} & N/A & N/A & $54.48 \pm 0.93\%$ & $71.31 \pm 0.78\%$\\
    TapNet \citep{Sung} & $61.65 \pm 0.15\%$ & $76.36 \pm 0.10\%$ & $63.08 \pm 0.15\%$ & $80.26 \pm 0.12\%$\\
    \textcolor{red}{LFD-ProtoNet(Ours)} & \textcolor{red}{$- \pm 0.32\%$} & \textcolor{red}{$76.5 \pm 0.10\%$} & \textcolor{red}{$-\pm 0.13\%$} & \textcolor{red}{$78.0 \pm 0.10\%$}\\
    \bottomrule
    \end{tabular}
\end{table}

\section{Conclusion}
We have proposed LFD-ProtoNet in few-shot learning problem settings.
Our method focuses on the covariance and mean of the support set. Such a feature extraction method is realized with LFDA, and the accuracy of LFD-ProtoNet improves by $-\%$ compared to TapNet which is the state-of-the- art variant of ProtoNet.
Moreover, the speed of the loss decreasing is much faster than that of TapNet and these result shows that LFDA extracts sufficient information to describe each class.
We theoretically explained that our feature extraction can maximize the expected risk bound in the $k$-shot learning. 
As in ProtoNet, LFD-ProtoNet is simple and easy to implement.

As our future work, we can add a pre-training step such as optimization of the initial parameters and we can consider the semi-supervised condition that we can also access some data without any annotation. These additional techniques are expected to further improve LFD-ProtoNet.
\newpage
 
\section*{Acknowledement}
MS was supported by JST CREST Grant Number JPMJCR18A2.
\bibliographystyle{plainnat}
\bibliography{neurips_2020_arxiv}

\newpage

\appendix
\section{Appendix}
\subsection{Derivation Details}
\label{ap:proof}
\begin{lemma}[Transformation of the representation vector]
\label{lem:lemma1}
We can obtain the following equation for $\overline{\phi(\mathcal{S}_c)}$ and $\overline{F\phi(\mathcal{S}_c)}$
\begin{align}
\overline{F\phi(\mathcal{S}_c)} =F\overline{\phi(\mathcal{S}_c)}
\end{align}
where $\mathcal{S}_c$ is the support set of the class $c\in \{\mathrm{a},\mathrm{b}\}$.
\end{lemma}
This is clear from the definition of $\overline{F\phi(\mathcal{S}_c)}$.
\begin{definition}[Nearest class score]
We define $\alpha$ as the difference between the distance of the query vector and the representation vector of class $\mathrm{b}$ and that of the query vector and the representation vector of class $\mathrm{a}$:
\begin{align}
\alpha=||F\phi(x) - F\overline{f(\mathcal{S}_{\mathrm{b}})}||^2 - ||F\phi(x)-F\overline{f(\mathcal{S}_{\mathrm{a}})}||^2.
\end{align}
\end{definition}

If $\alpha > 0$, then $F\phi(x)$ is closer to $F\overline{f(\mathcal{S}_{\mathrm{a}})}$ than $F\overline{f(\mathcal{S}_{\mathrm{b}})}$ and this implies $x$ belongs to the class $\mathrm{a}$.
Additionally, if $\alpha < 0$, then we can estimate $x$ belongs to the class $\mathrm{b}$. 

Cao et al. \citep{Tianshi} showed the following lemma.
\begin{lemma}[One-side Chebyshev's inequality for nearest class score]
By Chebyshev's inequality for $alpha$, the following inequality holds:
\begin{align*}
    \mathrm{Pr}(\alpha > 0) \geq \frac{\mathbb{E}[\alpha]}{\mathrm{Var}[\alpha] + \mathbb{E}[\alpha]}.
\end{align*}
\end{lemma}
Thus, when the query is in the class $\mathrm{a}$, the expected risk $R_{\mathcal{T},c}(\phi)$ defined in Def.\,\ref{def:ex_risk} is as follows:
\begin{align*}
    R_{\mathcal{T},c}(\phi)=1-\mathrm{Pr}(\alpha>0).
\end{align*}
Then we obtain
\begin{align*}
    R_{\mathcal{T},c}\leq 1-\frac{\mathbb{E}[\alpha]}{\mathrm{Var}[\alpha] + \mathbb{E}[\alpha]}.
\end{align*}
To bound the risk, we have to show the conditional expectation and the normal expectation of $\alpha$. The whole statement is as follows, and the proof of this lemma is derived afterwards.
\begin{lemma}[Conditional expectation of $\alpha$]
\label{lem:lemma2}
Consider k-shot learning. If $\Sigma_a=\Sigma_b$ and \\$p(\phi(X)|Y(X)=c)=\mathcal{N}(\mu_c,\Sigma_c)$, then
\begin{align}
\mathbb{E}_{x,S|\mathrm{a},\mathrm{b}}[\alpha] &= (\mu_a-\mu_b)^{\top}F^{\top}F(\mu_{\mathrm{a}}-\mu_{\mathrm{b}})\\
\mathbb{E}_{x,S,\mathrm{a},\mathrm{b}}[\alpha]&=2\mathrm{Tr}(\Sigma_F).
\end{align}
\end{lemma}
\begin{proof}[Proof of Lemma \ref{lem:lemma2}]

We compute the expectation of $\alpha=||F\phi(x) - \overline{F\phi(\mathcal{S}_{\mathrm{b}})}||^2 - ||F\phi(x) - \overline{F\phi(\mathcal{S}_{\mathrm{a}})}||^2$. We are given $\mathrm{a}$ and $\mathrm{b}$ drawn from the class distribution thus the conditional expectation is written as:

\begin{align}
\mathbb{E}_{x,S|\mathrm{a},\mathrm{b}}[\alpha] &= \underbrace{\mathbb{E}_{x,S|\mathrm{a},\mathrm{b}}[||F\phi(x) - \overline{F\phi(\mathcal{S}_{\mathrm{b}})}||^2]}_{\displaystyle\mathrm{(I)}} - \underbrace{\mathbb{E}_{x,S|\mathrm{a},\mathrm{b}}[||F\phi(x)-\overline{F\phi(\mathcal{S}_{\mathrm{a}})}||^2]}_{\displaystyle\mathrm{(II)}}.
\end{align}
The first term $\mathrm{(I)}$ is 
\begin{align}
\mathrm{(I)} &= \mathrm{Tr}(\Sigma_{F\phi(X)-\overline{F\phi(\mathcal{S}_{\mathrm{b}})}}) + \mathbb{E}_{x,S|\mathrm{a},\mathrm{b}}[F\phi(x)-\overline{F\phi(\mathcal{S}_{\mathrm{b}})}]^{\top}\mathbb{E}_{x,S|\mathrm{a},\mathrm{b}}[F\phi(x)-\overline{F\phi(\mathcal{S}_{\mathrm{b}})}],
\end{align}
where we use the relation $\mathbb{E}[||X||^2]= \mathrm{Tr}(\mathrm{Var}[X]) + \mathbb{E}[X]^{\top}\mathbb{E}[X]$. We compute the term $\Sigma_{F\phi(X)-\overline{F\phi(\mathcal{S}_b)}}$ as
\begin{align}
\Sigma_{F\phi(X)-\overline{F\phi(\mathcal{S}_{\mathrm{b}})}} =& \mathrm{Var}[F\phi(X)-\overline{F\phi(\mathcal{S}_{\mathrm{b}})}]\notag\\
=&\mathbb{E}_{x,\mathcal{S}|\mathrm{a},\mathrm{b}}[(F\phi(X)-\overline{F\phi(\mathcal{S}_{\mathrm{b}})}(F\phi(X)-\overline{F\phi(\mathcal{S}_{\mathrm{b}})})^{\top}] \notag\\
&- (F\mu_{\mathrm{a}} - F\mu_{\mathrm{b}})(F\mu_{\mathrm{a}}-F\mu_{\mathrm{b}})^{\top}\notag\\
=&F\left(\mathbb{E}[(\phi(X)-\overline{\phi(\mathcal{S}_{\mathrm{b}})})(\phi(X)-\overline{\phi(\mathcal{S}_{\mathrm{b}})})^{\top}] - (\mu_a-\mu_b)(\mu_a-\mu_b)^{\top}\right)F^{\top}\notag\\
=&F\left(\Sigma_c + \mu_a\mu_a^{\top} +\frac{1}{k}\Sigma_c+\mu_b\mu_b^{\top}-\mu_a\mu_b^{\top}-\mu_b\mu_a^{\top}-(\mu_a - \mu_b)(\mu_a-\mu_b)^{\top}\right)F^{\top}\notag\\
=&F\left(1+\frac{1}{k}\right)\Sigma_cF^{\top},
\end{align}
where $\mathbb{E}_{x,\mathcal{S}|\mathrm{a},\mathrm{b}}[\phi(x)]=\mu_{\mathrm{a}}$ because we assume $x$ belongs to the class $\mathrm{a}$, and we use the relation $\mathbb{E}_{x,\mathcal{S}|\mathrm{a},\mathrm{b}}[\overline{\phi(\mathcal{S}_{\mathrm{b}}}]=\mu_b$.
Thus, we obtain the following equation by applying the trace function
\begin{align}
\mathrm{(I)} = \mathrm{Tr}\left(F\left(1+\frac{1}{k}\right)\Sigma_cF^{\top}\right) + (\mu_a - \mu_b)^{\top}F^{\top}F(\mu_a -\mu_b).
\end{align}
In the same way, we obtain for $\mathrm{\mathrm{(II)}}$
\begin{align}
\mathrm{(II)} = \mathrm{Tr}\left(F\left(1+\frac{1}{k}\right)\Sigma_cF^{\top}\right).
\end{align}
We derive the following equation by subtracting $\mathrm{(II)}$ from $\mathrm{(I)}$ 
\begin{align}
\mathbb{E}_{x,S|\mathrm{a},\mathrm{b}}[\alpha] &= (\mu_{\mathrm{a}} - \mu_{\mathrm{b}})^{\top} F^{\top} F(\mu_{\mathrm{a}} -\mu_{\mathrm{b}}).
\end{align}
Then, we take an expectation as for $\mathrm{a}$ and $\mathrm{b}$; thus, the expectation of $\alpha$ is
\begin{align}
\mathbb{E}_{x,\mathcal{S},\mathrm{a},\mathrm{b}}[\alpha] &= \mathbb{E}_{\mathrm{a},\mathrm{b}}[(\mu_a-\mu_b)^{\top}F^{\top}F(\mu_a-\mu_b)]\notag\\
&=\mathbb{E}_{\mathrm{a},\mathrm{b}}[\mu_a^{\top}F^{\top}F\mu_a+\mu_b^{\top}F^{\top}F\mu_b-\mu_a^{\top}F^{\top}F^{\top}\mu_b-\mu_b^{\top}F^{\top}F\mu_b]\notag\\
&=\mathrm{Tr}(\Sigma_F)+\mathrm{Tr}(\Sigma_F)+2\mu^{\top}F^{\top}F\mu - 2\mu^{\top}F^{\top}F\mu\notag\\
&=2\mathrm{Tr}(\Sigma_F).
\end{align}
\end{proof}

From now on, we show the expectation of $\alpha$, and then we show its conditional variance. The whole statement is as follows and to prove it, we use the transformation theorem of the covariance matrix.
\begin{lemma}[Conditional variance of $\alpha$]
\label{lem:lemma3}
Under the same condition and notation as Lemma2, the following inequality holds
\begin{align}
\mathrm{Var}[\alpha,|\mathrm{a},\mathrm{b}]\leq 8\left(1+\frac{1}{k}\right)\mathrm{Tr}\left(\Sigma_{F,c}\left(\left(1+\frac{1}{k}\right)\Sigma_{F,c}+2\Sigma_F\right)\right).
\end{align}
\end{lemma}
\begin{proof}[Proof of Lemma \ref{lem:lemma3}]

We bound the conditional variance of $\alpha$ such that
\begin{align}
\mathrm{Var}[\alpha|\mathrm{a},\mathrm{b}] &= \mathrm{Var}[||F\phi(X)-\overline{F\phi(\mathcal{S}_{\mathrm{b}})}||^2-||F\phi(X)-\overline{F\phi(\mathcal{S}_{\mathrm{a}})}||^2]\notag\\
&=\mathrm{Var}[||F\phi(X)-\overline{F\phi(\mathcal{S}_{\mathrm{b}})}||^2] + \mathrm{Var}[||F\phi(X)-\overline{F\phi(\mathcal{S}_{\mathrm{a}})}||^2]\notag\\
&- 2\mathrm{Cov}(||F\phi(X)-\overline{F\phi(\mathcal{S}_{\mathrm{b}})}||^2,||F\phi(X)-\overline{F\phi(\mathcal{S}_{\mathrm{a}})}||^2)\notag\\
&\leq \mathrm{Var}[||F\phi(X)-\overline{F\phi(\mathcal{S}_{\mathrm{b}})}||^2] + \mathrm{Var}[||F\phi(X)-\overline{F\phi(\mathcal{S}_{\mathrm{a}})}||^2]\notag\\
&+2\sqrt{\mathrm{Var}[||F\phi(X)-\overline{F\phi(\mathcal{S}_{\mathrm{b}})}||^2]\mathrm{Var}[||F\phi(X)-\overline{F\phi(\mathcal{S}_{\mathrm{a}})}||^2]}\notag\\
&\leq 2\mathrm{Var}[||F\phi(X)-\overline{F\phi(\mathcal{S}_{\mathrm{b}})}||^2] + 2\mathrm{Var}[||F\phi(X)-\overline{F\phi(\mathcal{S}_{\mathrm{a}})}||^2].
\end{align}
When we compute these formulae, we use the relation $p+q\geq 2\sqrt{pq}$.
Then, we use the result on the normal distribution $\mathcal{N}(\mu,\Sigma)$ from Rencher and Schaalje \cite{Alvin},
\begin{thm}[Variance to trace \cite{Alvin}]
Considering random vector $y\sim \mathcal{N}(\mu,\Sigma)$ and symmetric matrix constants $Q$, we have:
\begin{align}
\mathrm{Var}[y^{\top}Qy] = 2\mathrm{Tr}((Q\Sigma)^2) + 4\mu^{\top}Q\Sigma Q\mu.
\end{align}
\end{thm}
Using this theorem, we take $F\phi(X)-\overline{F\phi(\mathcal{S}_c)}$ ($c\in\{\mathrm{a},\mathrm{b}\})$ as $y$ and $F$ as $Q$, and we use the previous result $\mathrm{Var}[F\phi(X)-\overline{F\phi(X)}]=F\left(1+\frac{1}{k}\right)\Sigma_cF^{\top}$. Thus, we obtain
\begin{align}
\mathrm{Var}[||F\phi(X)-\overline{F\phi(\mathcal{S}_{\mathrm{b}})}||^2] &= 2\left(1+\frac{1}{k}\right)^2\mathrm{Tr}(\Sigma_{F,c}^2)+4\left(1+\frac{1}{k}\right)(\mu_{\mathrm{a}}-\mu_{\mathrm{b}})^{\top}F^{\top}F(\mu_{\mathrm{a}}-\mu_{\mathrm{b}}),\\
\mathrm{Var}[||F\phi(X)-\overline{F\phi(\mathcal{S}_{\mathrm{a}})}||^2] &= 2\left(1+\frac{1}{k}\right)^2\mathrm{Tr}(\Sigma_{F,c}^2).
\end{align}
Finally, we obtain
\begin{align}
\mathbb{E}_{\mathrm{a},\mathrm{b}}\alpha|\mathrm{a},\mathrm{b}] &\leq  \mathbb{E}_{\mathrm{a},\mathrm{b}}\left[2\mathrm{Var}[||F\phi(X)-\overline{F\phi(\mathcal{S}_{\mathrm{b}})}||^2]+2\mathrm{Var}[||F\phi(X)-\overline{F\phi(\mathcal{S}_{\mathrm{b}})}||^2]\right]\notag\\
&=\mathbb{E}_{\mathrm{a},\mathrm{b}}\left[8\left(1+\frac{1}{k}\right)^2\mathrm{Tr}(\Sigma_{F,c}^2)+8\left(1+\frac{1}{k}\right)(\mu_{\mathrm{a}}-\mu_{\mathrm{b}})^{\top}F^{\top}\Sigma_{F,c}F(\mu_{\mathrm{a}}-\mu_{\mathrm{b}})\right]\notag\\
&= 8\left(1+\frac{1}{k}\right)\mathbb{E}_{\mathrm{a},\mathrm{b}}\left[\mathrm{Tr}\left\{\left(1+\frac{1}{k}\right)\Sigma_{F,c}^2+\Sigma_{F,c}F(\mu_{\mathrm{a}}-\mu_{\mathrm{b}})(\mu_{\mathrm{a}}-\mu_{\mathrm{b}})^{\top}F^{\top}\right\}\right]\notag\\
&=8\left(1+\frac{1}{k}\right)\mathrm{Tr}\left\{\Sigma_{F,c}\left(\left(1+\frac{1}{k}\right)\Sigma_{F,c}+2\Sigma_F\right)\right\}.
\end{align}
\end{proof}

\begin{proof}[Proof of Theorem \ref{maintheorem}]

\begin{align*}
    Var(\alpha) &=\mathbb{E}_{\mathrm{a},\mathrm{b},x,\mathcal{S}_c}[\alpha^2] - \mathbb{E}_{\mathrm{a},\mathrm{b},x,\mathcal{S}_c}[\alpha]^2\\
    &=\mathbb{E}_{\mathrm{a},\mathrm{b}}[\mathrm{Var}[\alpha|\mathrm{a},\mathrm{b}]+\mathbb{E}_{x,\mathcal{S}_c}[\alpha|\mathrm{a},\mathrm{b}]^2]-\mathbb{E}_{\mathrm{a},\mathrm{b},x,\mathcal{S}_c}[\alpha].
\end{align*}
Then we can write the bound of the expected risk as the following inequality:
\begin{align*}
    R_{\mathcal{T},c}\leq 1-\frac{\mathbb{E}[\alpha]}{\mathbb{E}_{\mathrm{a},\mathrm{b}}[\mathrm{Var}[\alpha|\mathrm{a},\mathrm{b}]+\mathbb{E}_{x,\mathcal{S}_c}[\alpha|\mathrm{a},\mathrm{b}]^2]}.
\end{align*}
By using Lemmas.\,\ref{lem:lemma2} and \ref{lem:lemma3}, we derive Theorem.\,\ref{maintheorem}.
\end{proof}
\subsection{Generalization for multi class}
\label{ap:multi}
We extend the inequality bounding the expected risk to a multi-class version. We assume that the label of the query data is $y\in \mathcal{Y}$. Let $\alpha_{y'}=||F\phi(x) - F\overline{f(\mathcal{S}_{y'})}||^2 - ||F\phi(x)-F\overline{f(\mathcal{S}_y)}||^2$. If $\forall y'\in \mathcal{Y},y'\neq y,\alpha_{y'}>0$, then the predicted label $\widehat{y}=y$. Thus the probability to predict the quey label correctly is equal to $\mathrm{Pr}(\cap_{y'\neq y}\alpha_{y'}>0)$. Therefore by using Frechet’s inequality, we obtain
\begin{align*}
    \mathrm{Pr}(\cap_{y'_\neq y}\alpha_{y'}>0) \geq \sum_{y'\neq y}\mathrm{Pr}(\alpha_{y'}>0) - (C-2).
\end{align*}
The expected risk is $R_{\mathcal{T},c}=1-\mathrm{Pr}(\cap_{y'\neq y})$; then, it can be bounded as
\begin{align*}
    R_{\mathcal{T},c}\leq& (C - 1) - \sum_{y'\neq y}\mathrm{Pr}(\alpha_{y'}>0)\\
     =& \sum_{y'\neq y}(1-\mathrm{Pr}(\alpha_{y'}>0))\\
     \leq& \sum_{i=1,i\neq y}^C\left(1-\frac{4\mathrm{Tr}\left(\Sigma_F\right)^2}{8\left(1+\frac{1}{k}\right)^2\mathrm{Tr}\left(\Sigma_{F,c}^2\right)+16\left(1+\frac{1}{k}\mathrm{Tr}(\Sigma_F \Sigma_{F,c})\right) + \mathbb{E}[\left((\mu_y-\mu_i)^{\top}F^{\top}F(\mu_y-\mu_i)\right)^2]}\right),
\end{align*}
where we denote the number of classes by $C$, use the expected risk bound for the binary classification case, and assume for any $c\in \mathcal{Y}$, $\Sigma_{F,c}$ is constant.

\subsection{Comparison}
We compare LFD-ProtoNet, TapNet, ProtoNet, and a network called FDA-ProtoNet, where we replace LFDA in LFD-ProtoNet with simple FDA, shown in Table \ref{tab:comp}. Their essential difference is the way they extract features. 
ProtoNet does not extract features from the output of the network, i.e., the number of features is equal to the number of the output vector dimensions. 
TapNet has a feature extractor that aims to reduce the misalignment of the vector, and the number of the dimensions removed by SVD is equal to the number of classes; then, it can reduce at most the number of classes dimensions, and these dimensions are much fewer than the total dimensions.
LFD-ProtoNet has a feature extractor that aims to extract the subspace to minimize the trace of $Sb^{-1}Sw$ and the number of directions or features extracted is equal to the number of all samples in the support set. If we use FDA instead of LFDA, the features are too few to learn the network. FDA-ProtoNet has few features extracted from the support set, but its fatal disadvantage is that it can not obtain sufficient amounts of information to classify the query vectors.
\begin{table}[!t]
\caption{Comparison of models}
\centering
\renewcommand{\arraystretch}{1.2}
\begin{tabular}{l|c|c|c}
\toprule
     Model &  Feature extraction & Representation & Features amount \\
     \hline
     LFD-ProtoNet & LFDA & Mean & Few and sufficient\\
     TapNet & SVD & Reference vector & Many and sufficient\\
     ProtoNet & None & Meant & Many and sufficient \\
     FDA-ProtoNet & FDA & Mean & Few and insufficient 
\end{tabular}
\label{tab:comp}

\end{table}

\subsection{Learning framework based on ProtoNet}
In this paper, we consider the algorithms based on ProtoNet shown in Algorithm\,\ref{fig: base}.
We are given a task set $\{T_u\}_{u=1}^N$ for one update process and each $T_u$ is composed of the support set $D_{\mathrm{s}}$ and query set $D_{\mathrm{q}}$. At each loop step of the task set, using a feature extractor $\mathrm{FE}$, we obtain one matrix $F$ that projects embedded vectors to another Euclidean space. Then, we compute representation vectors $\{\overline{f_{\theta}(x_{c,\cdot})}\}_{c=1}^C$ for each class, and by using these vectors and $F$, the empirical loss is computed with the softmax loss function. Finally,  the parameter $\theta$ is updated. This process is the whole framework based on ProtoNet. With this framework, we can distinguish ProtoNet, TapNet, and LFD-ProtoNet.
For instance, since simple ProtoNet extracts no features, $F=I_m$ and it computes the representation vectors $\{\overline{f_{\theta}(x_{c,\cdot})}\}_{c=1}^C$ by using the average in each class.
TapNet extracts features to reduce the misalignment of the vectors using SVD and the representation vectors are learnable parameters $\phi$.
LFD-ProtoNet extracts features with LFDA to find a subspace for embedded vectors to be projected to and the representation vectors are similar to ProtoNet, i.e, the mean vectors of support vectors projected by $F_{\mathrm{LFDA}}$.

\begin{algorithm}[!t]
    \renewcommand{\algorithmicrequire}{\textbf{Input:}}
    \renewcommand{\algorithmicensure}{\textbf{Output:}}
    \caption{Few-shot learning (k-shot) algorithm framework based on ProtoNet}
    \label{fig: base}
    \begin{algorithmic}[1]
    \REQUIRE training task $\{T\}_{u=1}^N \sim \mathcal{T}$ where $T_u=\{D_s,D_q\}$, $D_s=\{(x_{c,i},y_{c,i})\}_{1\leq c \leq C,1\leq i \leq k}$ and $D_q=\{(x'_{c,i},y'_{c,i})\}_{1\leq c \leq C,1\leq i \leq M}$.
    \STATE $L_{tr}\leftarrow 0$
    \FOR{$u$ in $u=1,2,\ldots,N$}
    \STATE $(D_s,D_q)\leftarrow T_u$
    \STATE $F=\mathrm{FE}(D_s,...)$
    \STATE Compute $\{\overline{f_{\theta}(x_{c,\cdot})}\}_{c=1}^C$
    \STATE $L_{T_u}\leftarrow 0$
    \FOR{$c$ in $c=1,2,\ldots,C$}
    \FOR{$i$ in $i=1,2,\ldots,M$}
    \STATE $L_{T_u} \leftarrow L_{T_u} + g(Ff_{\theta}(x'_{c,i}),\overline{f_{\theta}(x_{c,\cdot})},\{\overline{f_{\theta}(x_{s,\cdot})}\}_{s\neq c})$
    \ENDFOR
    \ENDFOR
    \STATE $L_{tr}\leftarrow L_{tr} + \frac{1}{CM}L_{T_u}$
    \ENDFOR
    \STATE $L_{tr}\leftarrow \frac{1}{N}L_{tr}$
    \STATE update $\theta$ with $L_{tr}$
    \end{algorithmic}
\end{algorithm}

\subsection{Local Fisher Discriminant Analysis}\label{Sec:LFDA}
\begin{figure}[t!]
    \begin{minipage}{0.5\hsize}
    \centering
    \includegraphics[width=7cm]{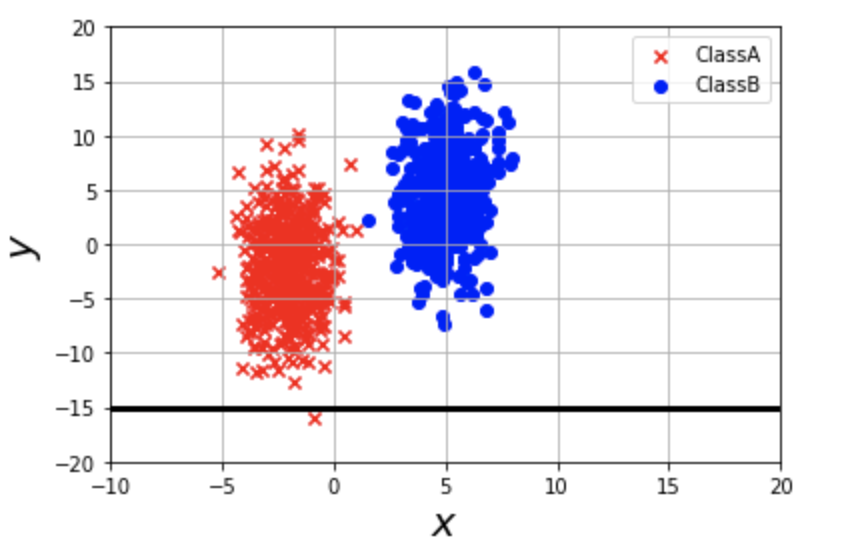}
    \end{minipage}
    \begin{minipage}{0.5\hsize}
    \centering
    \includegraphics[width=7cm]{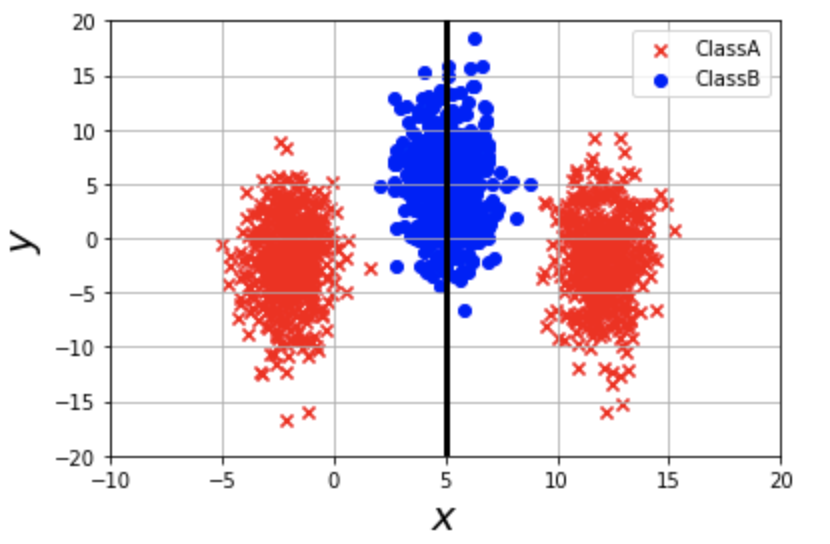}
    \end{minipage}
    \caption{Example of FDA: Left is a separable case. The variance of the axis $y$ is large; thus, FDA returns the projection subspace (black line) to make it small. Right is not a separable case. Class A (red) is separated, and FDA returns the projection subspace (black line) to make the within-covariance of class A small. }
    \label{fig:fda_ex_1}
\end{figure}
\begin{figure}[!t]
\centering
\includegraphics[width=6cm]{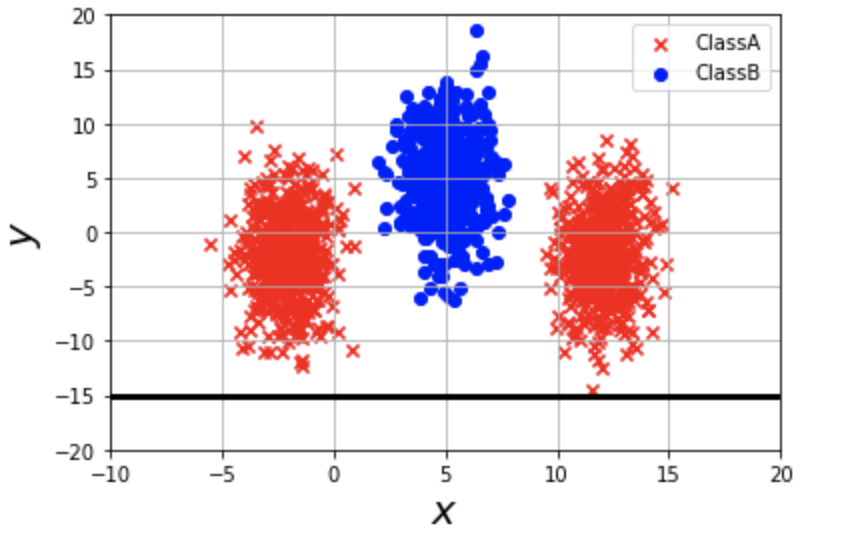}
\caption{Example of LFDA: The case FDA can not separate in Fig.\,\ref{fig:fda_ex_1}, can be separated by LFDA. The left side of Class A (red) and its right side are regarded as not similar so LFDA returns the projection subspace (black line) that makes the local within-covariance in each cluster small and the local between-covariance large.}
\label{fig:lfda_ex}
\end{figure}
First, we rewrite the previous notation of $S_{\mathrm{wit}}$ and $S_{\mathrm{bet}}$ as follows.
\begin{align}
S_{\mathrm{wit}} &= \frac{1}{2} {\displaystyle \sum_{c1,c2,i,j}} P_{c1,c2}^w (x_{c1,i}-x_{c2,j})(x_{c1,i}-x_{c2,j})^{\top},\quad
P_{c1,c2}^w =
\begin{cases}
\frac{1}{k} & (c1=c2)\\
0 & (c1 \neq c2)
\end{cases}\notag \\
S_{\mathrm{bet}} &= \frac{1}{2} {\displaystyle \sum_{c1,c2,i,j} P_{c1,c2}^b} (x_{c1,i}-x_{c2,j})(x_{c1,i}-x_{c2,j})^{\top},\quad
P_{c1,c2}^b = \begin{cases}
\frac{1}{k} - \frac{1}{kC} & (c1=c2)\\
\frac{1}{kC} & (c1\neq c2).
\end{cases}
\end{align}
Then, we add the similarity term $A_{c1,c2,i,j}\in[0,1]$ that means the similarity between $x_{c1,i}$ and $x_{c2,j}$ to $P_{i,j}^w$ and $P_{i,j}^b$, and the above equations are rewritten as
\begin{align}
S_{\mathrm{wit}}^{A} &= \frac{1}{2} {\displaystyle \sum_{c1,c2,i,j}} P_{c1,c2}^w (x_{c1,i}-x_{c2,j})(x_{c1,i}-x_{c2,j})^{\top},\quad
P_{i,j}^w =
\begin{cases}
\frac{A_{c1,c2,i,j}}{k} & (c1=c2)\\
0 & (c1 \neq c2)
\end{cases}\notag\\
S_{\mathrm{bet}}^{A} &= \frac{1}{2} {\displaystyle \sum_{c1,c2,i,j} P_{c1,c2}^b} (x_{c1,i}-x_{c2,j})(x_{c1,i}-x_{c2,j})^{\top},\quad
P_{i,j}^b = \begin{cases}
A_{c1,c2,i,j}\left(\frac{1}{kC} - \frac{1}{k}\right) & (c1=c2)\\
\frac{1}{kC} & (c1\neq c2).
\end{cases}
\end{align}
The squared exponential kernel $\exp(-(||x_{c_1,i}-x_{c_2,j}||)^2)$ is used for $A_{c_1,c_2,i,j}$.

To use FDA, we can separate samples shown in Fig.\,\ref{fig:fda_ex_1}(left) but not samples shown in Fig.\,\ref{fig:fda_ex_1}(right).
This is because in the latter case one class is sandwiched by the other class and FDA tries to make the within-class small even if samples in the same class are separated.

With LFDA, such complicated cases can be separated because the distant clusters in the same class are treated like different classes by the affinity term. The projection direction is illustrated in Fig.\,\ref{fig:lfda_ex}.

\subsection{Additional experiment for covariance}
\begin{figure}[t!]
    \begin{minipage}{0.5\hsize}
    \centering
    \includegraphics[width=7cm]{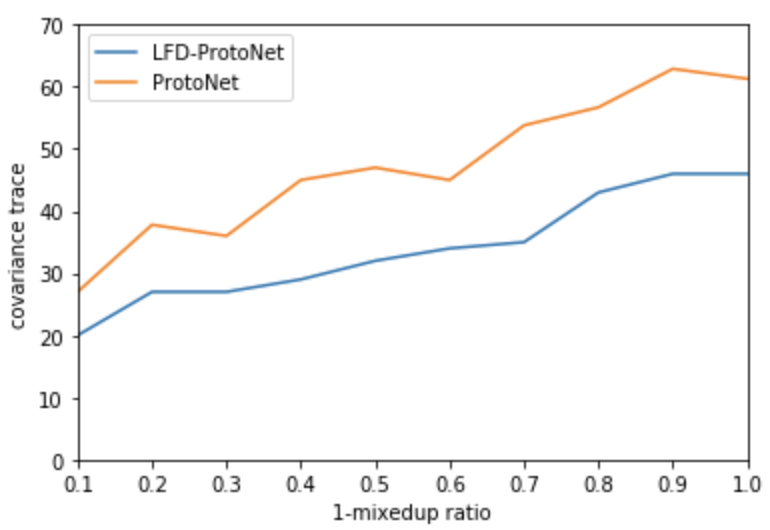}
    \end{minipage}
    \begin{minipage}{0.5\hsize}
    \centering
    \includegraphics[width=7cm]{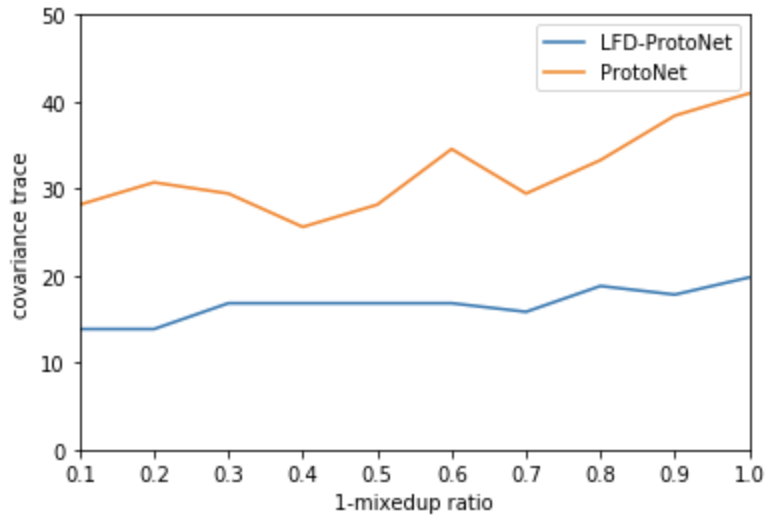}
    \end{minipage}
    \caption{Covariance ratio change: The horizontal axis is $1-\lambda$ and the vertical axis is the trace of covariance ratio. The left  and right graphs are tieredImageNet and miniImageNet cases, respectively.}
    \label{fig: cov_ratio}
\end{figure}
As we show, $\Sigma_F^{-1}\Sigma_{F,c}$ is desired to be small in LFD-ProtoNet. As in Cao et al. \citep{Tianshi}, $\Sigma^{-1}\Sigma_c$ is also desired to be small. Thus, we conducted an experiment to investigate how the covariance matrix ratio $\Sigma_F^{-1}\Sigma_{F,c}$ and $\Sigma^{-1}\Sigma_c$ changed when we changed the within-class covariance and between-class covariance by adding an operation to input images. We mixed input images $x$ and one image $x_{\mathrm{mix}}$ with a ratio $\lambda \in [0,1)$:
\begin{align*}
    x_{\mathrm{mixed},\lambda} = (1-\lambda)x + \lambda x_{\mathrm{mix}}.
\end{align*}
When $\lambda=0$, this means that we use normal images without any mix operations and when $\lambda \sim 1$, this means that whole data locate nearby $x_{\mathrm{mix}}$. We varied $\lambda$ from $0$ to $0.9$ and investigated how the covariance ratio changes with $\lambda$. The result is shown in Fig.\,\ref{fig: cov_ratio}. As we showed in Sec.\,\ref{sec:theory}, small covariance ratio is desirable and for any $\lambda$, the covariance ratio in LFD-ProtoNet is smaller than that of ProtoNet. We can conclude that LFD-ProtoNet performs better than ProtoNet.  The covariance ratio of ProtoNet also changes more drastically than that of LFD-ProtoNet, which means that LFD-ProtoNet is more stable than ProtoNet for the change of covariance. This is because, LFD-ProtoNet projected embedded vectors to the subspace that minimizes within-class covariance and maximizes between-class covariance thus projected vectors are less affected by the mixup operation.

\subsection{Notations}
\begin{longtable}{ll}
\hline
 & \multicolumn{1}{c}{Notations} \\
\hline\hline
\endfirsthead
\hline
 & \multicolumn{1}{c}{Notations} \\
\hline\hline
\endhead
$\mathcal{D}$ & A data distribution over $\mathcal{X}\times \mathcal{Y}$\\
$(x_{c,i},c)$ & A data sample of class $c$\\
$D_s$ & Support set drawn from $\mathcal{D}$\\
$D_q$ & Query set drawn from $\mathcal{D}$\\
$T$ & A task consists of support set $D_s$ and query set $D_q$\\
$\mathcal{T}$ & Task distribution \\
$\theta$ & Model parameters\\
$f_{\theta}$ & A network parameterized by $\theta$\\
$\overline{f_{\theta}(x_{c,\cdot})}$ & Representation vector of class $c$\\
$g$ & Loss function\\
$L_{\mathcal{T}}(f)$ & Generalization loss for the task distribution $\mathcal{T}$\\
$\widehat{L}_{\mathcal{T}}(f)$ & Empirical loss for the task distribution $\mathcal{T}$\\
$f^*$ & A optimal network minimizing the generalization loss $L_{\mathcal{T}}(f)$\\
$\widehat{f^*}$ & A network minimizing the empirical loss $\widehat{L}_{\mathcal{T}}(f)$\\
$F$ & A feature projection matrix\\
$\Sigma_c$ & A within-class covariance matrix of class $c$\\
$\Sigma$ & A between-class covariance matrix\\
$\Sigma_{F,c}$ & A within-class covariance matrix of class $c$ where vectors are projected by $F$\\
$\Sigma_F$ & A between-class covariance matrix where vectors are projected by $F$\\
$C$ & Number of total classes\\
$M$ & Number of samples per class\\
$N$ & Number of tasks in the training step\\
$k$ & Number of shots
\end{longtable}
\end{document}